\documentclass{article}

\usepackage[margin=1in]{geometry}
\usepackage{times}

\usepackage[utf8]{inputenc}
\usepackage[T1]{fontenc}
\usepackage{microtype}

\usepackage{amsmath,amssymb,amsthm,amsfonts}
\newtheorem{theorem}{Theorem}

\usepackage{booktabs}
\usepackage{graphicx}
\graphicspath{{./}{../}{figures/}}
\usepackage{subcaption}
\usepackage{multirow}
\usepackage{algorithm}
\usepackage{algorithmic}

\usepackage[numbers,sort&compress]{natbib}
\usepackage{xcolor}
\usepackage{hyperref}
\hypersetup{
  colorlinks=true,
  linkcolor=blue!60!black,
  citecolor=blue!60!black,
  urlcolor=blue!60!black,
}

\title{\textbf{Cortex-Inspired Continual Learning:\\
Unsupervised Instantiation and Recovery of Functional Task Networks}}

\author{%
  Kevin McKee \quad
  Thomas Hazy \quad
  Yicong Zheng \quad
  Zacharie Bugaud \quad
  Thomas Miconi \quad \\
  Astera Institute
}

\date{}

\begin{document}
\maketitle

\begin{abstract}
Block-sequential continual learning demands that a single model both protect prior solutions from catastrophic forgetting and efficiently infer at inference time which prior solution matches the current input without task labels. We present \emph{Functional Task Networks} (FTN), a parameter-isolation method inspired by structural and dynamical motifs found in the mammalian neocortex. Similar to mixture-of-experts, this method uses a high dimensional, self-organizing binary mask over a large population of small but deep networks, inspired by dendritic models of pyramidal neurons. The mask is produced by a three-stage procedure: (1) gradient descent on a continuous mask identifies task-relevant neurons, (2) a smoothing kernel biases the result toward spatial contiguity, (3) and $k$-winner-take-all binarizes the resulting group at a fixed capacity budget. Like mixture-of-experts, each neuron is an independent deep network, so disjoint masks give exactly disjoint gradient updates, providing structural guarantees against catastrophic forgetting. This three-stage procedure recovers the subnetwork of a previously-trained task in a single gradient step, providing unsupervised task segmentation at inference time. We test it on three continual-learning benchmarks: (1) a synthetic multi-task classification/regression generator, (2) MNIST with shuffled class labels (pure concept shift), and (3) Permuted MNIST (domain shift). On all three, FTN with several, fine grained smoothing steps (FTN-Slow) results in nearly zero forgetting. FTN with a larger kernel and only 2 iterations of smoothing (FTN-Fast) trades off a small amount of retention for increased speed. We show that the spatial organization mechanism reduces the effective mask search from the combinatorial top-$k$ subset problem in $\mathcal{O}(\binom{H}{k})$ to the complexity of a near-linear scan in $\mathcal{O}(H)$ over compact cortical neighborhoods, which is parallelized by the gradient-based update.
\end{abstract}

\section{Introduction}
\label{sec:intro}

Modern learning systems, from lifelong agents to continually pretrained language models~\citep{gururangan2020dontstop,yildiz2024investigating}, must absorb new data without degrading or destroying old competence. When a network trained on task $A$ is subsequently fine-tuned on task $B$, gradient updates on $B$ overwrite the parameter configurations that encoded $A$'s solution~\citep{mccloskey1989catastrophic,french1999catastrophic}, producing \emph{catastrophic forgetting}. The problem is most acute under two complementary conditions: \emph{block-sequential} training, where each task is a contiguous data block with no interleaving (the standard regime for continual pretraining of LLMs); and \emph{concept shift}, where the same input features map to different correct outputs across tasks~\citep{gama2014concept,morenotorres2012unifying}. In both settings, the learner must also solve an \emph{unsupervised task detection} problem: given a few examples of input-output pairs at inference time, identify which prior solution (if any) applies~\citep{aljundi2019taskfree}.

Within continual learning, there are two kinds of data non-stationarity that result in forgetting. \emph{Domain shift} (or covariate shift) leaves the input-output relationship intact but changes the marginal input distribution~\citep{quinonero2009dataset,morenotorres2012unifying}, which can result in partial overwriting of solutions if the whole input domain is not represented in later updates. Classic benchmarks such as Permuted MNIST~\citep{goodfellow2013empirical,kirkpatrick2017overcoming,van2019three} are usually treated as domain-incremental benchmarks, although the raw-coordinate conditional changes under each fixed permutation. \emph{Concept shift} changes the input-output relationship itself. Identical features yield different correct labels~\citep{gama2014concept,morenotorres2012unifying}. Concept shift is the harder setting for regularization methods because it can demand that a shared model represent incompatible input-output maps, especially when task identity is not provided at test time~\citep{van2019three}. Block-sequential training amplifies the negative effects on retention from both kinds of shift. A further desideratum is \emph{task-free} continual learning~\citep{aljundi2019taskfree,parisi2019continual,van2019three}: the learner must figure out task boundaries from the data stream itself rather than being told.

Three families of methods have emerged to address forgetting~\citep{parisi2019continual}. \emph{Regularization} approaches~\citep{kirkpatrick2017overcoming,zenke2017continual} penalize updates to parameters estimated to be important for prior tasks; they typically fail under concept shift because a single parameter cannot simultaneously satisfy conflicting weight configurations. \emph{Experience replay}~\citep{rebuffi2017icarl,lopez2017gradient} stores and rehearses past examples to anchor shared representations, but the memory and compute footprint scales with task history. \emph{Parameter isolation}~\citep{mallya2018packnet,serra2018overcoming,wortsman2020supermasks,rusu2016progressive,yoon2018lifelong} assigns each task a disjoint subnetwork, converting the forgetting problem into a routing problem: if different tasks recruit different parameters, training one cannot disturb the other.

This paper contributes a parameter-isolation method, \emph{Functional Task Networks} (FTN), that unifies the three core requirements of continual learning in a single mechanism:
\begin{itemize}
  \item \textbf{Structural forgetting protection.} A parallel-neuron backbone stores each neuron's weights as a private tensor; a binary routing mask gates the neurons' scalar outputs. Disjoint masks therefore produce exactly disjoint gradient paths.
  \item \textbf{Rapid unsupervised task detection.} A three-stage mask \emph{configurer}, i.e., gradient descent on a continuous mask, lateral smoothing, and $k$-winner-take-all (KWTA) binarization, is run cold-started on a batch of input-output pairs with no task label. The same procedure that \emph{installs} a subnetwork during training \emph{recovers} it at inference time, in only one gradient step.
  \item \textbf{Efficient shared-neuron consolidation.} Where tasks do share neurons, the mask overlap between the current task's mask and buffered task masks provides a natural query for selective experience replay, minimizing the amount of replay needed.
\end{itemize}

\subsection{Inspirations from biology}
Our method takes inspiration from several structural and dynamical motifs observed in the mammalian neocortex, namely (1) dendritic computation in pyramidal neurons~\citep{poirazi2003pyramidal,poirazi2020illuminating,beniaguev2021single}, (2) lateral excitatory and horizontal connectivity among cortical pyramidal neurons~\citep{hubel1962receptive,gilbert1983clustered,douglas2004neuronal}, (3) lateral inhibition among neurons resulting in sparse task representation~\citep{douglas2004neuronal,maass2000computational,amari1977dynamics}, (4) spatio-temporal attractor dynamics and task-specific ensemble recruitment~\citep{yang2019task,cichon2015branch,mongillo2018inhibitory,rolls2010attractor,khona2022attractor}, and (5) basal-ganglia and dopaminergic modulation of cortical gating and selection~\citep{mink1996basal,schultz1997neural,frank2001interactions,oreilly2006making}. In this paper, we will give a brief exposition of these influences but leave more rigorous treatment of neuroscientific theory to subsequent work.

First, the dendritic model of pyramidal neurons holds that each neuron can have nonlinear processing capacity comparable to a multi-layer artificial neural network, owing to its complex network of dendritic compartments~\citep{poirazi2003pyramidal,poirazi2020illuminating,beniaguev2021single}. We therefore specified each ``neuron'' or expert in our model as a deep multi-layer perceptron (MLP) with relatively small hidden dimensionality and a scalar output, akin to the emission rate of action potentials. The functional benefit of this model is that deep, nonlinear feature extraction is compartmentalized within neurons, with neural outputs linearly combined for particular tasks. This simplifies the problem of selecting optimal features for a particular task to a search over linear combinations. Furthermore, the restricted output dimensionality discourages neurons from learning redundant solutions, a problem related to expert collapse and load imbalance in mixture-of-experts models~\citep{shazeer2017outrageously,fedus2022switch}, as no one neuron is structurally capable of solving the higher-dimensional task alone. Incorporating stochasticity in the output via dropout~\citep{srivastava2014dropout} is an algorithmic analogue of neural response variability~\citep{faisal2008noise,mainen1995reliability} and prevents the network from relying too much on any single or joint contribution of neurons to the task, granting tolerance to the imperfect recovery of optimal neural assemblies.

Each neuron sits at a location on a two-dimensional grid. In the cortex, nearby neurons interact through local recurrent and horizontal connectivity~\citep{hubel1962receptive,gilbert1983clustered,douglas2004neuronal}, which we implement as a simple, uniform and positive-valued convolution, i.e., a smoothing kernel. Rather than applying this smoothing kernel to the neural activations directly, we apply it to the multiplicative routing mask such that neurons group together based on their relevance to a task, rather than their relevance to a particular input. This spatial grouping serves two purposes in our model. First, it draws networks apart spatially, protecting them from overlap and interference. Second, smoothing reduces the search space of possible $k$-sized subnetworks over $N$ neurons from $\binom{H}{k}$ solutions to as few as $H$ solutions under the idealized case of a single fixed-size contiguous blob. However, the actual number of solutions depends on the contiguity within the networks. That contiguity is determined by the size of the smoothing kernel and the number of smoothing steps. A smaller kernel with few smoothing steps permits each task network to form as many small but co-active groups.

Aside from lateral excitation, neurons also connect laterally to inhibitory interneurons, which introduce competition~\citep{maass2000computational,douglas2004neuronal,amari1977dynamics}. We implement this using a common abstraction for lateral inhibition, namely $k$-winner-take-all (KWTA or top-k) selection. This choice results in a simple method of restricting the computational budget per sub-task. It is by no means the only way, and in fact implementing a negative-valued smoothing kernel larger than the excitatory kernel, (i.e., a ``mexican hat'' kernel~\citep{wilson1972excitatory,amari1977dynamics}) results in flexible allocation of neurons, whose approximate $k$ and variance around $k$ that can be tuned continuously. We leave investigation of such kernels, akin to cellular automata, to future work and instead prioritize in this work simple, sufficient, and rapid computations for machine learning purposes.

Finally, our use of gradient signals to select an initial population of best-performing neurons, which then acts as the initial condition for spatial consolidation and competition, is loosely inspired by basal-ganglia gating and dopaminergic reward-prediction signals in the control of cortical state selection~\citep{mink1996basal,schultz1997neural,frank2001interactions,oreilly2006making,rolls2010attractor}.

Each of these motifs was chosen primarily for its specific algorithmic benefit rather than to simulate biology exactly. Altogether, they produce a system that performs unsupervised instantiation and recruitment of artificial neural assemblies, enabling block-sequential continual learning over a highly flexible class of machine learning architectures.

\paragraph{Contributions.}
\begin{enumerate}
  \item We introduce a parallel-neuron architecture with gradient-driven, spatially organized routing masks (Section~\ref{sec:method}) and prove a structural no-forgetting property for disjoint masks.
  \item We show that the same mask selection process, applied cold-started to an batch of input-output pairs with no task label, recovers the correct prior-task subnetworks in as few as 1 gradient step, yielding unsupervised task segmentation at inference time (Sections~\ref{sec:method},~\ref{sec:experiments}).
  \item We evaluate the method on three benchmarks spanning synthetic multi-task data, pure concept shift (MNIST Shuffled Labels), and domain shift (Permuted MNIST). FTN outperforms regularization and ablation baselines, including KWTA without spatial dynamics and EWC, on all three (Section~\ref{sec:experiments}).
  \item We characterize when spatial organization helps (Section~\ref{sec:discussion}): it accelerates and stabilizes mask recovery by projecting the combinatorial top-$k$ problem onto the low-dimensional manifold of contiguous cortical blobs, while the combinatorial flexibility of plain KWTA remains available for out-of-distribution generalization, a direction we leave for future work due to its entailment of NP-hard problem structure.
\end{enumerate}

\section{Related Work}
\label{sec:related}

\paragraph{Regularization-based continual learning.}
Elastic Weight Consolidation (EWC)~\citep{kirkpatrick2017overcoming} penalizes changes to parameters with high diagonal Fisher information; Synaptic Intelligence~\citep{zenke2017continual} accumulates an online importance score. Both degrade when tasks conflict on a shared parameter, as we verify empirically on MNIST Shuffled Labels.

\paragraph{Parameter-isolation methods.}
PackNet~\citep{mallya2018packnet} iteratively prunes and freezes weights per task. Hard Attention to the Task (HAT)~\citep{serra2018overcoming} learns a per-task gated mask using an explicit task embedding and a gradient-reversal penalty. SupSup~\citep{wortsman2020supermasks} stores a superposition of per-task binary masks over a fixed random backbone and selects among them at inference by gradient alignment. Progressive Networks~\citep{rusu2016progressive} and DEN~\citep{yoon2018lifelong} instead grow fresh pathways per task. FTN differs from all of these in that (a) the mask is produced at reconfiguration time by a short gradient descent on the live model and current batch, without learning a mask or task embedding across tasks, and (b) the same procedure that installs a mask also retrieves it, so the mechanism is the same at train and test time and no task identity is required at inference.

\paragraph{Mixture-of-experts and sparse routing.}
A mask over independent per-neuron subnetworks also inherits the sparse-routing perspective of mixture-of-experts~\citep{shazeer2017outrageously,fedus2022switch}. FTN can be read as a sparsely-routed MoE in which the router is replaced by a short inner optimization and the experts are a 2-D cortical grid of small MLPs.

\section{Method}
\label{sec:method}

\subsection{Problem Setup}

We consider block-sequential continual learning with $N$ tasks. At block $t \in \{1,\ldots,N\}$ the learner receives a contiguous stream of samples $(x,y) \sim \mathcal{D}_t$, trains on them for a fixed budget, and must thereafter retain performance on all prior blocks without revisiting their data. At evaluation the task identity is not revealed; the model must recover whichever prior solution best explains a test batch~\citep{van2019three}.

\subsection{Parallel-Neuron Architecture}
\label{sec:backbone}

We place $H$ ``neurons'' on a square cortical grid of side $D=\sqrt{H}$, following the broad topographic-map intuition that nearby units can share or compete over related functions~\citep{kohonen1982self}. Each neuron $k \in \{1,\ldots,H\}$ is itself a small feedforward network of $L$ hidden layers with inner width $d_\mathrm{inner}$ and a scalar output. Let
\(
  \mathrm{MLP}_k(x) : \mathbb{R}^{d_\mathrm{in}} \to \mathbb{R}
\)
denote neuron $k$'s private computation; its weights $\{W^{(\ell)}_k, b^{(\ell)}_k\}_{\ell=1}^{L+1}$ are stored as the $k$-th slice of a $(H, \cdot, \cdot)$ tensor and applied via a batched einsum that broadcasts the shared input across the $H$ slots. A binary routing mask $m \in \{0,1\}^H$ gates the stacked scalar outputs, and a single linear readout $W_\mathrm{out} \in \mathbb{R}^{d_\mathrm{out}\times H}$ produces the prediction:
\begin{align}
  z_k  &= \mathrm{MLP}_k(x) \in \mathbb{R}, \quad k = 1,\ldots,H, \\
  \hat{y} &= W_\mathrm{out} \bigl(\mathrm{Dropout}(z) \odot m\bigr).
\end{align}

The resulting model is a sparse mixture of $H$ experts~\citep{shazeer2017outrageously}. Dropout on the output vector~\citep{srivastava2014dropout} further decorrelates the ensemble and makes the downstream readout robust to noise in which specific neurons are selected. Because each neuron owns a disjoint weight tensor, the only shared parameter across neurons is the readout $W_\mathrm{out}$; its columns are tied one-to-one to neuron outputs. A neuron with $m_k=0$ contributes zero to the readout, receives zero gradient at its column of $W_\mathrm{out}$, and---because its output is unused downstream---receives zero gradient throughout its private weight tensor. Disjoint masks produce exactly disjoint updates.

\begin{figure}[ht]
    \centering
    \includegraphics[width=0.75\linewidth]{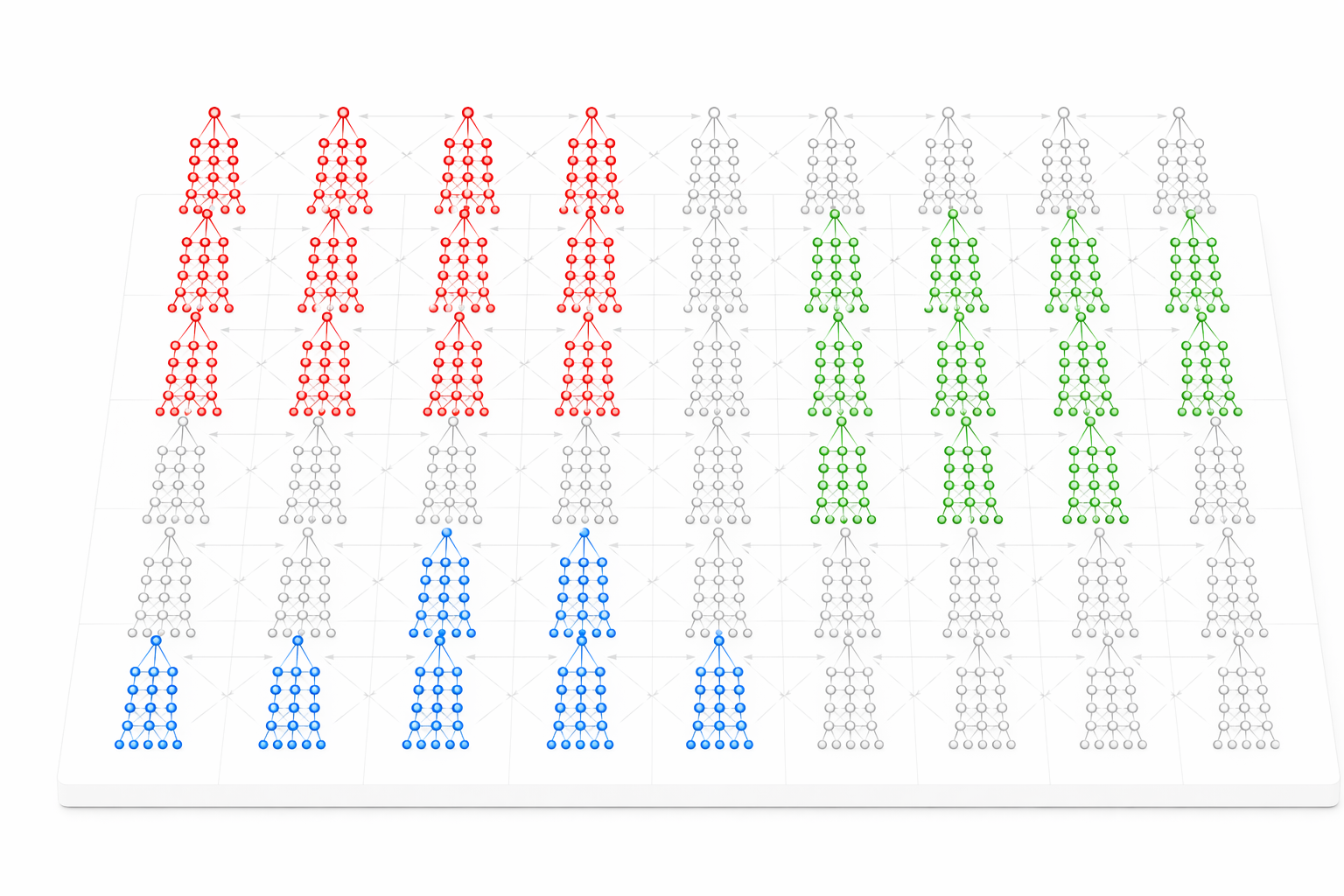}
    \caption{The functional task network model, visualized. Each color represents a separate, spatially cohesive subnetwork responsible for a particular task. Each neuron in the assembly is a multi-layer feed forward network, such that the task output is determined by a linear combination of many highly nonlinear basis functions.}
    \label{fig:visualization}
\end{figure}

\begin{theorem}[Structural forgetting guarantee\footnote{Implementations should be careful to prevent optimizer-state updates, decoupled weight decay, shared output biases, BatchNorm/statistics updates, and other non-gradient parameter changes on inactive neurons.}]
\label{thm:isolation}
Let $m^{(s)}, m^{(t)} \in \{0,1\}^H$ be the masks in force during training blocks $s < t$. If $m^{(s)} \odot m^{(t)} = \mathbf{0}$, then for every parameter $\theta_k$ internal to neuron $k$ with $m^{(s)}_k = 1$ and every column $[W_\mathrm{out}]_{:,k}$, the gradient of the task-$t$ loss with respect to $\theta_k$ is zero. Consequently, training on block $t$ leaves the task-$s$ solution exactly unchanged.
\end{theorem}

\begin{proof}
Direct from the forward pass: $z_k$ enters the readout only via $m^{(t)}_k \cdot z_k$. With $m^{(t)}_k = 0$, $z_k$ is multiplied by zero in the loss, so $\partial \mathcal{L}_t / \partial z_k = 0$; by the chain rule $\partial \mathcal{L}_t / \partial \theta_k = 0$ for every parameter on which $z_k$ depends, and $\partial \mathcal{L}_t / \partial [W_\mathrm{out}]_{:,k}$ is likewise zero.
\end{proof}

\subsection{Three-Stage Mask Configurer}
\label{sec:mask_config}

The mask is produced by a short procedure that takes the current model, a data batch, and nothing else (no task identity, no task embedding). It combines three cortical motifs.

\begin{algorithm}[t]
\caption{\textsc{SmoothKWTA} mask configurer}
\label{alg:mask_config}
\begin{algorithmic}[1]
\REQUIRE Model $f_\theta$, batch $(X,Y)$ or unlabeled batch $X$ with a self-supervised loss, grid side $D$, kernel size $s$, lateral steps $T$, winners $k$, reconfig steps $S$, reconfig LR $\eta$
\STATE $m \leftarrow \mathbf{0} \in \mathbb{R}^{D^2}$ \COMMENT{cold start}
\FOR{$i = 1,\ldots,S$}
  \STATE $\tilde{m} \leftarrow \sigma(m)$
  \STATE $\mathcal{L} \leftarrow \mathrm{Loss}\bigl(f_\theta(X, \tilde{m}),\, Y\bigr)$
  \STATE $m \leftarrow m - \eta\,\widehat{\nabla}_m \mathcal{L}$ \COMMENT{Adam step on mask logits}
\ENDFOR
\STATE Reshape $m$ to $D\times D$
\FOR{$t = 1,\ldots,T$}
  \STATE $m \leftarrow K_s * m$ \COMMENT{uniform $s\times s$ lateral convolution, torus padding}
\ENDFOR
\STATE $m \leftarrow \mathrm{KWTA}\bigl(\mathrm{flatten}(m),\, k\bigr)$
\RETURN $m$
\end{algorithmic}
\end{algorithm}

\paragraph{Stage 1, Gradient-based proposal.}
The mask is treated as a continuous parameter passed through a sigmoid, and is optimized for $S$ Adam steps to minimize the task loss on the current batch. The gradient is computed end-to-end through the NeuroModel with the relaxed mask in the forward pass. Because the mask is cold-started from zeros, each call is a fresh, task-agnostic probe of which neurons currently drive the loss. Biologically, this stage plays the role of a reward-prediction-error~\citep{schultz1997neural} signal selecting among competing neural assemblies.

\paragraph{Stage 2, Dense lateral excitation.}
The continuous mask is reshaped to the $D\times D$ cortical grid and convolved $T$ times with a uniform $s\times s$ kernel (torus padding). Nearby cells reinforce each other; noisy gradient cues are averaged out over the kernel footprint. This is a minimal model of the dense excitatory lateral connectivity observed in cortical micro-circuits~\citep{gilbert1983clustered,douglas2004neuronal}. Its algorithmic role is to project the mask onto the low-dimensional manifold of \emph{contiguous blobs}---a drastic reduction of the combinatorial solution space, which we analyze below.

\paragraph{Stage 3, $k$-winner-take-all (KWTA).}
The top-$k$ smoothed values are binarized to one, the rest to zero. KWTA~\citep{maass2000computational} abstracts the cortical lateral inhibition that enforces global competition among excitatory neurons~\citep{douglas2004neuronal}, giving each task a fixed capacity budget $k/H$.

\subsection{Training Procedure}
\label{sec:training}

Training proceeds block-sequentially. For each task block $t$ we reconfigure the mask once per training batch (Exp.~1) or at the start of each epoch (Exp.~2--3), then take an Adam step on the model weights while keeping the mask fixed. The Adam optimizer state is reset at every task boundary; this is needed for the disjoint-mask gradient guarantee of Theorem~\ref{thm:isolation} to hold over many steps, since carried-over first/second-moment estimates would otherwise emit non-zero updates on inactive neurons even when their gradients are exactly zero. After finishing the block we evaluate performance on all tasks $0..t$ under two protocols: \textbf{stored} (use the mask saved at training time, the oracle protocol) and \textbf{reconfig} (re-run the configurer cold-started on a fresh evaluation batch, the realistic protocol). All headline numbers in Section~\ref{sec:experiments} use the reconfig protocol.

\subsection{Mask Configuration Variants}
\label{sec:variants}

We consider six variants that share the backbone and differ only in the configurer (Table~\ref{tab:variants}). FTN-Fast and FTN-Slow are the two main lateral-dynamics settings; KWTA-only is the ablation that turns off lateral smoothing; FixedMask and NoMask are structural baselines; EWC is the classical regularization baseline.

\begin{table}[ht]
\centering
\caption{Mask configurer variants.}
\label{tab:variants}
\begin{tabular}{llrl}
\toprule
Name & Kernel & Steps & Role \\
\midrule
FTN-Fast   & $17\times17$ & 2  & Rapid spatial convergence via large kernel \\
FTN-Slow   & $3\times3$   & 15 & Fine-grained iterative spatial smoothing \\
KWTA-only  & --           & 0  & Ablation: gradient + KWTA, no lateral dynamics \\
FixedMask  & --           & -- & Static disjoint blocks; structural upper bound \\
NoMask     & --           & -- & All-ones mask; naive fine-tuning \\
EWC        & --           & -- & \citep{kirkpatrick2017overcoming}, $\lambda=400$ \\
\bottomrule
\end{tabular}
\end{table}

\section{Experiments}
\label{sec:experiments}

\subsection{Setup}
\label{sec:setup}

The parallel-neuron backbone is instantiated with $H{=}1024$ slots on a $32{\times}32$ grid; each slot is an $L{=}8$-layer MLP of inner width $d_\mathrm{inner}{=}8$; output dropout $p{=}0.2$. Readout is a single linear layer. Optimizer: Adam at $\eta{=}3{\times}10^{-4}$ with weight decay disabled, and the optimizer's first- and second-moment estimates are reset at every task boundary so that no momentum from a prior task can drift parameters whose gradient is exactly zero under the current mask (a precondition for Theorem~\ref{thm:isolation} to hold empirically; see footnote). KWTA sparsity $k{=}128$ (i.e.\ each task uses $12.5\%$ of the slots). All experiments use 8 seeds on a single H100; reported numbers are mean\,$\pm$\,std.

We report three standard continual-learning metrics~\citep{lopez2017gradient} computed from the final-block performance matrix: Average Accuracy (\textbf{ACC}; mean MSE for regression, lower is better), Forgetting Measure (\textbf{FM}), and Backward Transfer (\textbf{BWT}). For regression, FM is the average peak-MSE minus final-MSE on prior tasks (lower is better) and BWT is the average final-MSE minus first-trained-MSE on prior tasks (lower is better, since rising MSE signals forgetting). Headline numbers use the \emph{mask-recovery} protocol (task identity hidden at inference).

\subsection{Experiment 1: Synthetic Multi-Task Benchmark}
\label{sec:exp1}

\paragraph{Setup.}
Three synthetic tasks: a fixed encoder projects a 2-D input into a 24-D latent; each task selects a disjoint 8-D block via a task-specific hidden mask; a sinusoidal nonlinearity at frequency~8 produces the classification (binary) or regression (scalar) target. One epoch of 1000 training steps per task, batch~256, training batch size $B_{\text{train}}{=}256$. The mask is reconfigured \emph{every training batch} (cold-start from zeros, Algorithm~\ref{alg:mask_config}) using a separate reconfig batch of size $B_{\text{recfg}}{=}256$, and final-task evaluation is done over a fixed test set of $B_{\text{eval}}{=}4096$ samples per task. The classification and regression variants share an identical training pipeline (same backbone, same data generator, same number of epochs / steps / batches, same Adam model optimizer with state reset per task, same KWTA budget $k{=}128$, same zero mask cold-start). The two variants differ only in (i) the readout dimension, 2 logits with cross-entropy for classification, 1 scalar with MSE for regression, and (ii) the mask reconfigurer's inner optimization: $S{=}1$ step at $\eta_m{=}1.0$ for classification versus $S{=}10$ steps at $\eta_m{=}0.2$ for regression. The longer, gentler regression schedule was chosen empirically: the scalar MSE loss benefits from accumulating evidence over more, smaller steps before the lateral and KWTA stages binarize, whereas a single high-LR step is sufficient for classification.

\begin{figure}[ht]
    \centering
    \includegraphics[width=0.9\linewidth]{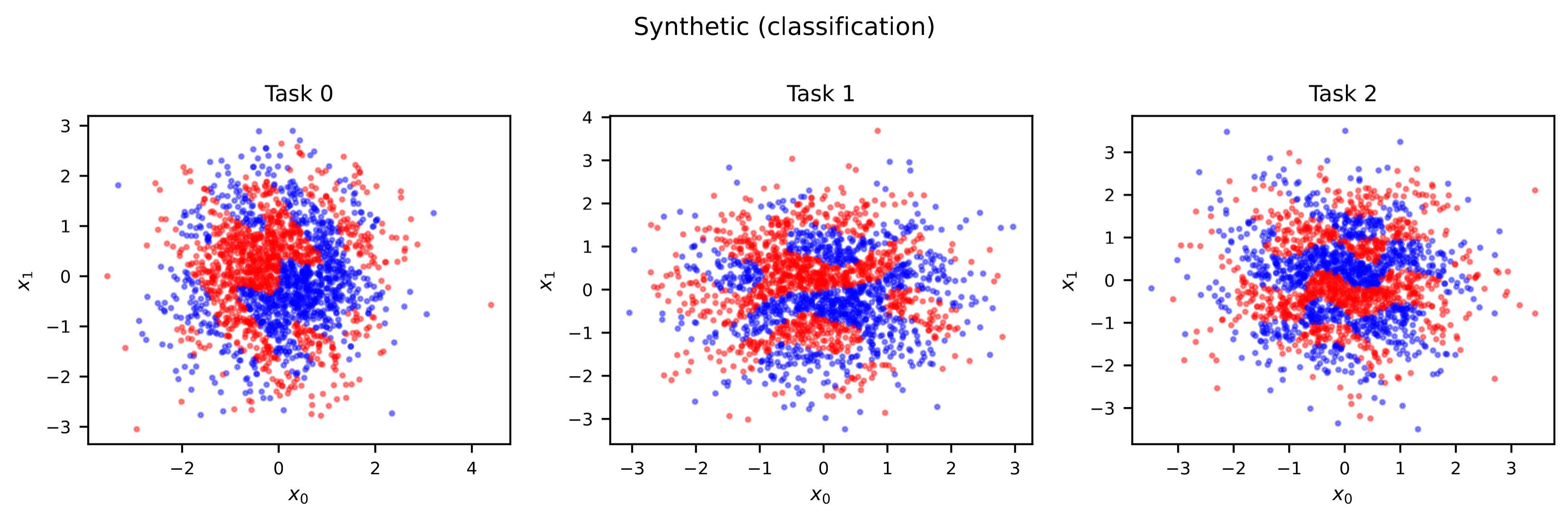}
    \caption{Example data distributions from the classification task generator, with 2 dimensional input mapping to 2 possible output classes. The model must learn all three classification mappings block-sequentially without forgetting any.}
    \label{fig:tasks}
\end{figure}

\paragraph{Results.}
Table~\ref{tab:exp1} reports mask-recovery metrics. On classification, FTN-Fast and FTN-Slow cleanly recover the task subnetworks ($\mathrm{ACC}{=}0.899 / 0.896$, $\mathrm{FM}{=}0.042 / 0.014$); FTN-Slow's smaller FM reflects tighter spatial separation of the trained subnetworks. KWTA-only does noticeably worse ($\mathrm{ACC}{=}0.640$, $\mathrm{FM}{=}0.152$) because its single-step gradient proposal is not stabilized by spatial averaging and tends to pick a different neuron set on each call. EWC retains the naive-fine-tune performance of NoMask ($\mathrm{ACC}{=}0.665$, $\mathrm{FM}{=}0.463$) because the three tasks conflict on the shared encoder. FixedMask is the structural upper bound by construction; with the optimizer-state reset described in Section~\ref{sec:setup} it attains $\mathrm{ACC}{=}0.938$ and $\mathrm{FM}{=}0.000\pm0.000$, exactly matching Theorem~\ref{thm:isolation}. On regression, the same qualitative ordering holds: FTN-Fast and FTN-Slow reduce error far below the unconstrained baselines ($\mathrm{MSE}{=}0.126 / 0.167$ for FTN-Fast / FTN-Slow vs.\ $0.589$ for KWTA-only), with FixedMask the structural lower bound ($\mathrm{MSE}{=}0.006$, $\mathrm{FM}{=}0.000$) and NoMask/EWC failing to retain ($\mathrm{MSE}{\approx}0.31$). FTN-Slow's higher MSE relative to FTN-Fast is paid for by the smallest FM among adaptive methods ($0.003$ vs.\ $0.013$): more iterations of the small kernel produce a tighter neuron-set separation, at the cost of slightly slower convergence on the current task.

\begin{table}[h]
\centering
\caption{Experiment 1 (Synthetic CL, 3 tasks; mask-recovery eval protocol).  Mean $\pm$ std over 8 seeds (clf) / 8 seeds (reg).  Regression columns report MSE (lower is better); the regression FM is the average (peak MSE $-$ final MSE) across earlier tasks (lower is better); BWT is the average (final MSE $-$ first-trained MSE) across earlier tasks (lower is better, since rising MSE signals forgetting).}
\label{tab:exp1}
\small
\begin{tabular}{l rrr rrr}
\toprule
 & \multicolumn{3}{c}{Classification} & \multicolumn{3}{c}{Regression} \\
\cmidrule(lr){2-4}\cmidrule(lr){5-7}
Method & ACC$\uparrow$ & FM$\downarrow$ & BWT$\uparrow$ & MSE$\downarrow$ & FM$\downarrow$ & BWT$\downarrow$ \\
\midrule
NoMask & $0.666\pm0.031$ & $0.466\pm0.046$ & $-0.466\pm0.046$ & $0.312\pm0.050$ & $0.054\pm0.071$ & $0.463\pm0.075$ \\
FixedMask & $0.938\pm0.014$ & $0.000\pm0.000$ & $0.000\pm0.000$ & $0.006\pm0.002$ & $0.000\pm0.000$ & $0.000\pm0.000$ \\
KWTA & $0.640\pm0.062$ & $0.152\pm0.052$ & $-0.146\pm0.061$ & $0.589\pm0.316$ & $0.028\pm0.038$ & $0.248\pm0.247$ \\
FTN-Fast & $0.899\pm0.037$ & $0.042\pm0.045$ & $-0.034\pm0.047$ & $0.126\pm0.057$ & $0.013\pm0.017$ & $0.051\pm0.046$ \\
FTN-Slow & $0.896\pm0.025$ & $0.014\pm0.025$ & $0.005\pm0.034$ & $0.167\pm0.075$ & $0.003\pm0.007$ & $0.067\pm0.070$ \\
EWC & $0.665\pm0.031$ & $0.463\pm0.046$ & $-0.463\pm0.046$ & $0.311\pm0.050$ & $0.054\pm0.071$ & $0.461\pm0.075$ \\
\bottomrule
\end{tabular}
\end{table}

\paragraph{Mask-overlap loss vs.\ recall-error loss.}
A single ACC/MSE number conflates two distinct failure modes: (i) at \emph{training time}, two tasks may overlap on shared neurons so the second overwrites part of the first, visible as a drop in the \emph{stored}-mask perf matrix relative to FixedMask; and (ii) at \emph{evaluation time}, the configurer may fail to re-find the trained mask, visible as a drop from \emph{stored} to \emph{recovered} in the same row. Taking each method's final-row average against the FixedMask oracle on prior tasks (FixedMask reference: clf~ACC$=0.939$ / reg~MSE$=0.006$), the train-time and recall-time gaps decompose as in Table~\ref{tab:overlap_recall}.

\begin{table}[h]
\centering
\caption{Decomposition of Exp.\ 1 mask-recovery loss into training-time mask overlap and evaluation-time recall error, relative to the FixedMask oracle. Classification: ACC drop on prior tasks (lower is closer to oracle). Regression: MSE excess on prior tasks (lower is closer to oracle). FTN-Slow's negative classification overlap reflects a small empirical advantage in recall accuracy on prior tasks at this seed budget.}
\label{tab:overlap_recall}
\small
\begin{tabular}{l rr rr}
\toprule
 & \multicolumn{2}{c}{Classification ACC drop} & \multicolumn{2}{c}{Regression MSE excess} \\
\cmidrule(lr){2-3}\cmidrule(lr){4-5}
Method & Mask overlap & Recall error & Mask overlap & Recall error \\
\midrule
NoMask    & $+0.430$ & $+0.000$ & $+0.461$ & $+0.000$ \\
EWC       & $+0.428$ & $+0.000$ & $+0.459$ & $+0.000$ \\
KWTA      & $+0.154$ & $+0.131$ & $+0.272$ & $+0.267$ \\
FTN-Fast  & $+0.016$ & $+0.010$ & $+0.098$ & $+0.032$ \\
FTN-Slow  & $-0.004$ & $+0.021$ & $+0.110$ & $+0.022$ \\
\bottomrule
\end{tabular}
\end{table}

The two columns must be read jointly. NoMask and EWC look ``perfect on recall'' only because they never had a separable mask to begin with, they pay everything to overlap. KWTA-only sits in between: training-time overlap is modest, but its single-step gradient proposal does not always converge to the same neuron set the next time it is run, so it loses an additional $+0.131$ ACC on classification and $+0.267$ MSE on regression to recall error alone. FTN-Fast and FTN-Slow are the only methods that keep both columns small simultaneously. FTN-Slow has the smallest training-time overlap on classification (slightly better than the FixedMask oracle on prior tasks at this seed budget, hence the negative entry) and the smallest recall error on regression among adaptive methods; FTN-Fast trades a small amount of overlap protection for a faster current-task fit. See Section~\ref{app:perf_matrices} for the full per-method stored--vs--recovered matrices on both classification and regression.

Figure~\ref{fig:exp1_masks} visualizes the FTN-Fast mask allocations across seeds as an RGB overlay (one color per task): the lateral kernel groups neurons into compact, non-overlapping blobs that are topographically consistent across random initialisations.

\begin{figure}[ht]
\centering
\begin{subfigure}[b]{0.48\textwidth}
  \includegraphics[width=\textwidth]{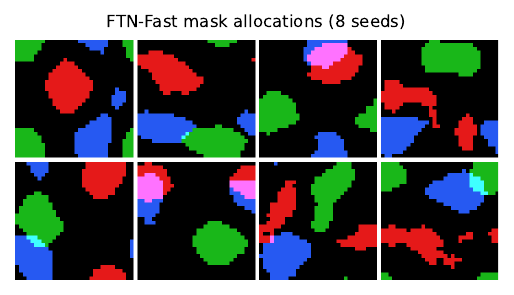}
  \caption{Classification masks (FTN-Fast, 8 seeds)}
\end{subfigure}
\hfill
\begin{subfigure}[b]{0.48\textwidth}
  \includegraphics[width=\textwidth]{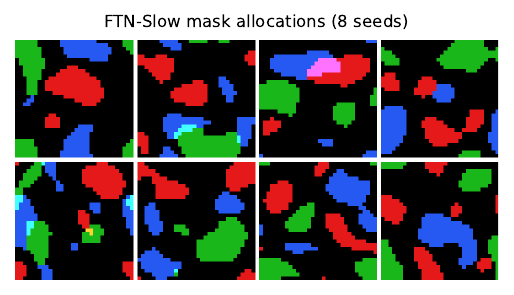}
  \caption{Classification masks (FTN-Slow, 8 seeds)}
\end{subfigure}
\caption{RGB mask allocations across 8 random seeds for the synthetic benchmark. Each tile is one seed; each color is one of the three tasks. Note the spatially compact, approximately non-overlapping subnetworks. Overlap is less in FTN-Slow, leading to better retention.}
\label{fig:exp1_masks}
\end{figure}

\subsection{Experiment 2: MNIST with Shuffled Class Labels}
\label{sec:exp2}

\paragraph{Setup.}
Five tasks on standard MNIST. Task~0 uses the canonical digit labels; tasks~1--4 apply different random permutations of the 10-class label set. This is \emph{pure} concept shift: the images are identical across tasks, only the labels change~\citep{van2019three,morenotorres2012unifying}. Five epochs $\times$ 400 steps per task, batch 256. Reconfiguration per epoch ($S{=}20$, $\eta_m{=}0.2$).

\paragraph{Results.}
Table~\ref{tab:exp2} is a clean demonstration of the concept-shift failure mode of regularization: NoMask and EWC both collapse to near-chance ($\mathrm{ACC}=0.255$, $\mathrm{FM}=0.904$), because the single shared network cannot represent the five incompatible label-permutation mappings on the same weights. All masking methods succeed. FTN-Slow (fine-grained kernel) is best ($\mathrm{ACC}=0.976\pm0.001$, $\mathrm{FM}=0.004\pm0.001$); KWTA-only ($0.951 / 0.025$) is next; FTN-Fast ($0.938 / 0.052$) trails slightly with higher variance. Five-task concept shift is fully resolved in under $1\%$ forgetting by the best FTN variant.

\begin{table}[h]
\centering
\caption{Experiment 2: MNIST Shuffled Labels (concept shift, 5 tasks).  Mean $\pm$ std over 8 seeds; mask-recovery eval protocol.}
\label{tab:exp2}
\begin{tabular}{l rrr}
\toprule
Method & ACC$\uparrow$ & FM$\downarrow$ & BWT \\
\midrule
NoMask & $0.255\pm0.000$ & $0.904\pm0.001$ & $-0.904\pm0.001$ \\
KWTA & $0.951\pm0.004$ & $0.025\pm0.005$ & $-0.023\pm0.006$ \\
FTN-Fast & $0.938\pm0.051$ & $0.052\pm0.064$ & $-0.052\pm0.064$ \\
FTN-Slow & $0.976\pm0.001$ & $0.004\pm0.001$ & $-0.004\pm0.001$ \\
EWC & $0.255\pm0.000$ & $0.903\pm0.001$ & $-0.903\pm0.001$ \\
\bottomrule
\end{tabular}
\end{table}

\begin{figure}[ht]
\centering
\begin{subfigure}[b]{0.45\textwidth}
  \includegraphics[width=\textwidth]{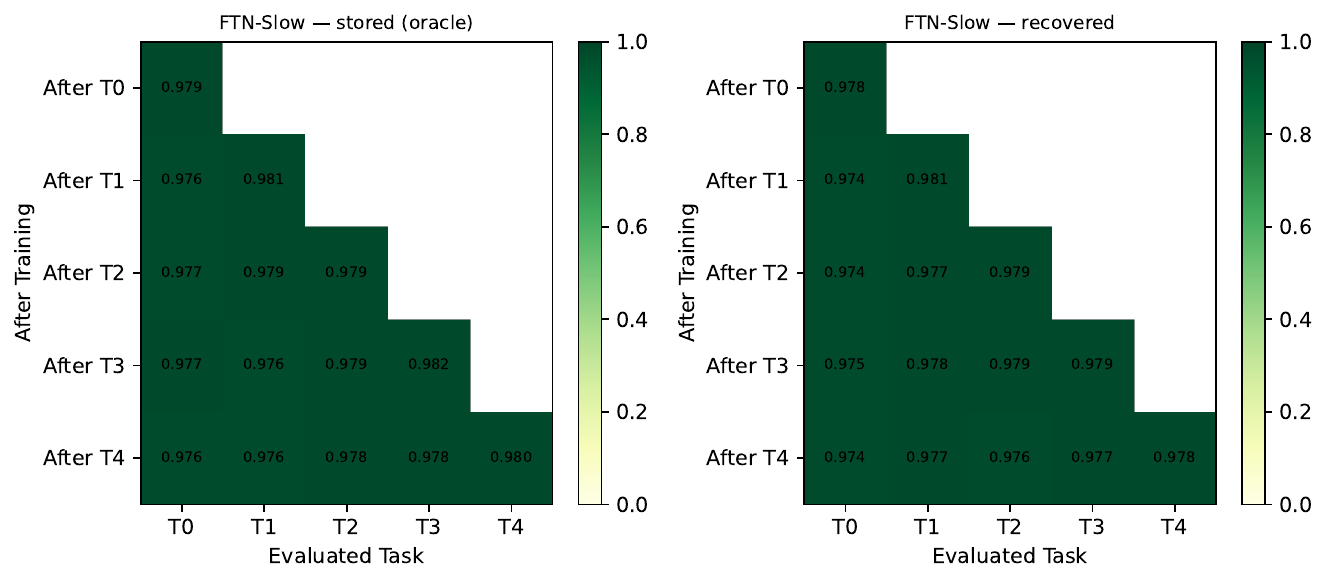}
  \caption{FTN-Slow: stored~$|$~recovered}
\end{subfigure}
\hfill
\begin{subfigure}[b]{0.45\textwidth}
  \includegraphics[width=\textwidth]{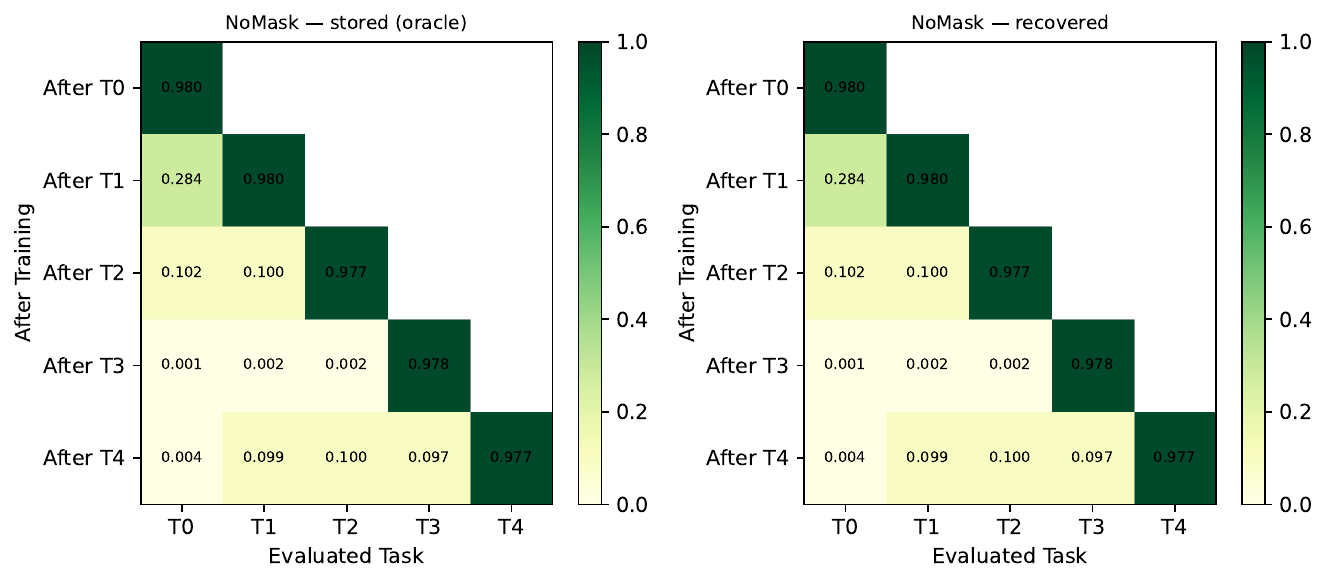}
  \caption{NoMask}
\end{subfigure}
\caption{Performance matrices for MNIST Shuffled Labels. Cell $(i,j)$ is test accuracy on task~$j$ after training through task~$i$. FTN-Slow retains all five label permutations; NoMask loses each prior task immediately upon training the next.}
\label{fig:exp2_matrix}
\end{figure}

\subsection{Experiment 3: Permuted MNIST}
\label{sec:exp3}

\paragraph{Setup.}
The classical CL benchmark~\citep{goodfellow2013empirical,kirkpatrick2017overcoming,van2019three}: 10 tasks, each a random pixel permutation of MNIST; labels unchanged. This is usually treated as domain-incremental learning: the label space and task objective are shared, although the raw-coordinate mapping from pixels to digits changes under each permutation, so a single well-regularized readout can do reasonably well. Three epochs $\times$ 400 steps per task, batch 256. Reconfiguration per epoch ($S{=}10$, $\eta_m{=}0.3$).

\paragraph{Results.}
Table~\ref{tab:exp3} shows the regime where regularization methods hold up: EWC reaches $\mathrm{ACC}=0.941$, $\mathrm{FM}=0.037$. FTN still performs well, with FTN-Slow best ($\mathrm{ACC}=0.959\pm0.002$, $\mathrm{FM}=0.019\pm0.002$), FTN-Fast close behind ($0.956 / 0.023$), KWTA-only lower ($0.926 / 0.042$), and NoMask trailing substantially ($0.829 / 0.164$). The ranking FTN-Slow~$>$~FTN-Fast~$>$~KWTA-only~$>$~EWC~$>$~NoMask is consistent across both MNIST-style benchmarks.

\begin{table}[h]
\centering
\caption{Experiment 3: Permuted MNIST (domain shift, 10 tasks).  Mean $\pm$ std over 8 seeds; mask-recovery eval protocol.}
\label{tab:exp3}
\begin{tabular}{l rrr}
\toprule
Method & ACC$\uparrow$ & FM$\downarrow$ & BWT \\
\midrule
NoMask & $0.829\pm0.014$ & $0.164\pm0.016$ & $-0.164\pm0.016$ \\
KWTA & $0.926\pm0.008$ & $0.042\pm0.009$ & $-0.036\pm0.010$ \\
FTN-Fast & $0.956\pm0.006$ & $0.023\pm0.006$ & $-0.023\pm0.006$ \\
FTN-Slow & $0.959\pm0.002$ & $0.019\pm0.002$ & $-0.018\pm0.002$ \\
EWC & $0.941\pm0.003$ & $0.037\pm0.003$ & $-0.037\pm0.003$ \\
\bottomrule
\end{tabular}
\end{table}

\subsection{Cross-Benchmark Summary}

Figure~\ref{fig:metric_summary_mnist} (MNIST Shuffled) and the tables above share a single qualitative pattern: the six methods cluster into three regimes. (1) \emph{Ablation controls} (NoMask, EWC) forget catastrophically under concept shift and partially under domain shift. (2) \emph{Adaptive masking without spatial structure} (KWTA-only) prevents forgetting reliably but has lower mask-recovery accuracy and higher variance than FTN. (3) \emph{Spatial-masking FTN} (FTN-Fast, FTN-Slow) is uniformly at or near the structural upper bound. FTN-Slow has the lowest forgetting (FM) on every benchmark, while FTN-Fast's larger kernel converges faster on the current task and therefore gives slightly better current-task fit on the synthetic regression variant. The choice between them is best read as a bias--variance tradeoff between recall stability and task-fit speed.

\begin{figure}[ht]
\centering
\includegraphics[width=0.75\textwidth]{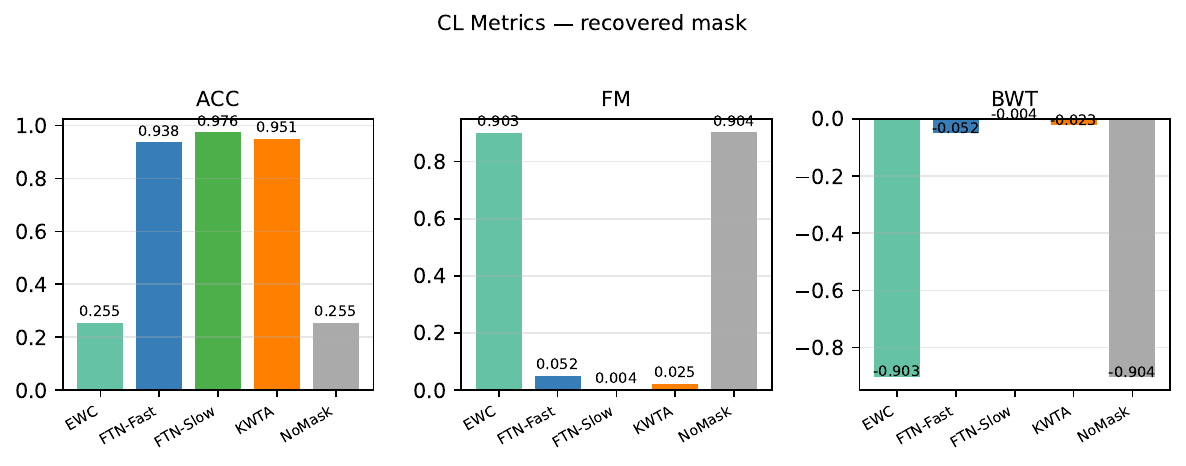}
\caption{MNIST Shuffled Labels (mask-recovery protocol, 8 seeds). FTN variants recover prior-task solutions almost perfectly; KWTA-only is strong but slightly behind; NoMask/EWC are at chance.}
\label{fig:metric_summary_mnist}
\end{figure}

\section{Discussion}
\label{sec:discussion}

\paragraph{Summary of the empirical finding.}
The experiments support a narrow but useful conclusion: in block-sequential continual learning, a parameter-isolated backbone can be paired with a short, task-identity-free mask-configuration procedure to recover prior solutions without being given an oracle task label. FTN is strongest in the concept-shift regimes, where the same input features require different outputs across blocks and shared-parameter methods are expected to fail~\citep{gama2014concept,morenotorres2012unifying,van2019three}. This is the setting in which regularization methods such as EWC and SI are least well matched, because penalizing weight movement cannot make a single shared parameter simultaneously implement incompatible mappings~\citep{kirkpatrick2017overcoming,zenke2017continual}. The empirical pattern is consistent across the benchmarks: NoMask forgets, EWC helps mainly when the task sequence is closer to domain shift than pure concept shift, KWTA-only shows that sparse mask selection already supplies much of the benefit, and the spatially smoothed FTN variants improve recovery stability and retention.

\paragraph{Why FTN forgets less.}
Theorem~\ref{thm:isolation} gives the basic mechanism. If two tasks use disjoint masks, training on one task induces no gradient on the private weights or readout columns used by the other. FTN therefore turns forgetting into a routing problem, as in parameter-isolation methods such as PackNet, HAT, SupSup, Progressive Networks, and DEN~\citep{mallya2018packnet,serra2018overcoming,wortsman2020supermasks,rusu2016progressive,yoon2018lifelong}. The difference is that FTN does not require a task embedding or an oracle task identity at evaluation. The same optimization over mask logits that installs a subnetwork during training is run again on a fresh support batch to recover the mask. Thus the relevant comparison is not only to methods that protect old parameters, but to methods that can identify which protected parameters should be active when the task identity is hidden~\citep{aljundi2019taskfree,van2019three}.

\paragraph{What spatial organization contributes.}
The spatial kernel should be understood as an inductive bias on the mask search, not as a formal guarantee that the original combinatorial problem has disappeared. Without smoothing, KWTA-only can in principle choose any of the $\binom{H}{k}$ masks, a sparse subset-selection problem that is hard in the worst case~\citep{natarajan1995sparse}. With smoothing, the search is biased toward compact regions on the cortical grid. In the idealized case of a single fixed-size contiguous blob, the number of candidate placements is only $\mathcal{O}(H)$; with multiple blobs, irregular shapes, or weak smoothing, the effective search space is larger. The experiments show that this bias is useful for \emph{recall}: when the goal is to recover a previously trained task network, smoothing turns noisy single-batch gradients into stable attractor-like mask patterns and reduces variance relative to KWTA-only. The price is reduced combinatorial flexibility, so spatial FTN should not be expected to dominate KWTA-only in every setting. Tasks that require novel recombinations of distant units, or OOD settings where no previously learned compact region is sufficient, may benefit from the larger search space retained by KWTA-only.

\paragraph{FTN-Slow versus FTN-Fast.}
The consistent advantage of FTN-Slow over FTN-Fast suggests that, in these benchmarks, the useful role of the lateral stage is not merely to form a compact mask. A large kernel with few iterations can impose compactness quickly, but it also averages away local distinctions in the gradient proposal. A small kernel applied over many iterations appears to produce a cleaner basin of attraction: neighboring high-relevance neurons reinforce one another while weak isolated responses are suppressed. This interpretation matches the observed ordering, where FTN-Slow is closest to the FixedMask upper bound, FTN-Fast trades some retention for speed, and KWTA-only is competitive but less stable. It also clarifies the practical hyperparameter choice: kernel size and smoothing depth control a bias--variance tradeoff in mask recovery, rather than a purely biological modeling choice.

\paragraph{Relationship to replay and mixture-of-experts.}
FTN is compatible with replay, but the reported results do not rely on it. The mask overlap $\rho_j=(m_{\mathrm{cur}}\cdot m_j)/\|m_{\mathrm{cur}}\|_1$ could provide a query for selective replay: examples from prior tasks whose masks overlap the current mask are precisely the examples most likely to constrain shared neurons. This would connect FTN to replay methods such as iCaRL and GEM~\citep{rebuffi2017icarl,lopez2017gradient}, but it remains future work. FTN is also close in spirit to sparse mixture-of-experts models~\citep{shazeer2017outrageously,fedus2022switch}. The main difference is where routing comes from. Standard MoE learns a feed-forward router over inputs; FTN performs an inner-loop optimization over a mask using the current batch loss. This makes routing slower, but it decouples task selection from a learned input-only router and lets the same mechanism be used for both training-time mask installation and evaluation-time mask recovery.

\paragraph{Limitations.}
Several limitations remain. First, the present benchmarks are small and largely synthetic. Second, the baseline set is intentionally minimal: NoMask, EWC, KWTA-only, and FixedMask isolate the contribution of spatial mask recovery, but potential comparisons include related baselines such as HAT, PackNet, SupSup, GEM, and iCaRL under matched capacity and compute~\citep{mallya2018packnet,serra2018overcoming,wortsman2020supermasks,lopez2017gradient,rebuffi2017icarl}, though these methods differ in exactly what they are designed to accomplish, e.g., requiring task-labels or explicit store-recall decisions.
Third, the architecture spends capacity by allocating $k$ of $H$ neurons to each recovered subnetwork. Long sequences of mutually incompatible tasks will eventually exhaust the grid unless $k$ is adapted, masks are allowed to share safely, or new neurons are added. Fourth, the structural forgetting guarantee applies to masked gradient paths; implementations must also avoid unintended changes through optimizer state, decoupled weight decay, shared biases, normalization statistics, or other non-gradient updates. Finally, mask configuration adds inference-time compute and introduces kernel, $k$, and reconfiguration-step hyperparameters whose scaling behavior is not yet known.

\paragraph{Outlook.}
The most important next experiments are therefore clear. FTN should be tested on longer task sequences with capacity pressure, under matched comparisons to stronger task-free and parameter-isolation baselines, and with support-batch size varied to quantify how much evidence is needed for reliable mask recovery. A second direction is to combine FTN with selective replay or adaptive neurogenesis so that overlap can be used constructively rather than avoided. A third direction is to replace the supervised configuration loss with a self-supervised objective in settings where such a loss is already native, such as continual language-model pretraining. The present results justify FTN as a strong recall mechanism for block-sequential concept shift; whether the same mask-search machinery can support open-ended transfer and OOD recombination remains the central open question.

\section{Conclusion}

We introduced \emph{Functional Task Networks}, a continual-learning method that combines a parallel bank of independent per-neuron subnetworks with a cortically-organized three-stage mask configurer. The backbone gives a strong structural forgetting guarantee (Theorem~\ref{thm:isolation}); the configurer rapidly installs and recovers task-specific subnetworks on the cortical grid, turning the mask-selection problem from unconstrained combinatorial search ($\binom{H}{k}$) toward a near-linear search over spatially contiguous blob locations ($\mathcal{O}(H)$ under the single-blob idealization). On three benchmarks spanning concept shift, domain shift, and realistic tabular data, FTN-Slow attains the best reported numbers among adaptive methods: $\mathrm{FM}\le 0.02$ on MNIST Shuffled Labels, and Permuted MNIST, all under the realistic mask-recovery protocol with task identity hidden at inference time. The method is compatible with block-sequential training, requires no task labels, and provides a natural mask-overlap query for minimizing experience replay. The combinatorial flexibility that plain KWTA retains remains a promising basis for future work on knowledge transfer and out-of-distribution generalization, which we frame as a distinct, computationally harder problem whose co-integration with efficient recall is a key open question for continual learning.

\bibliographystyle{plainnat}
\bibliography{refs}

@article{mccloskey1989catastrophic,
  title={Catastrophic interference in connectionist networks: The sequential learning problem},
  author={McCloskey, Michael and Cohen, Neal J},
  journal={Psychology of learning and motivation},
  volume={24},
  pages={109--165},
  year={1989},
  publisher={Elsevier}
}

@article{french1999catastrophic,
  title={Catastrophic forgetting in connectionist networks},
  author={French, Robert M},
  journal={Trends in cognitive sciences},
  volume={3},
  number={4},
  pages={128--135},
  year={1999},
  publisher={Elsevier}
}

@article{kirkpatrick2017overcoming,
  title={Overcoming catastrophic forgetting in neural networks},
  author={Kirkpatrick, James and Pascanu, Razvan and Rabinowitz, Neil and
          Veness, Joel and Desjardins, Guillaume and Rusu, Andrei A and
          Milan, Kieran and Quan, John and Ramalho, Tiago and Grabska-Barwinska,
          Agnieszka and others},
  journal={Proceedings of the national academy of sciences},
  volume={114},
  number={13},
  pages={3521--3526},
  year={2017},
  publisher={National Acad Sciences}
}

@inproceedings{zenke2017continual,
  title={Continual learning through synaptic intelligence},
  author={Zenke, Friedemann and Poole, Ben and Ganguli, Surya},
  booktitle={International Conference on Machine Learning},
  pages={3987--3995},
  year={2017},
  publisher={PMLR}
}

@article{lopez2017gradient,
  title={Gradient episodic memory for continual learning},
  author={Lopez-Paz, David and Ranzato, Marc'Aurelio},
  journal={Advances in neural information processing systems},
  volume={30},
  year={2017}
}

@inproceedings{rebuffi2017icarl,
  title={iCaRL: Incremental classifier and representation learning},
  author={Rebuffi, Sylvestre-Alvise and Kolesnikov, Alexander and Sperl, Georg
          and Lampert, Christoph H},
  booktitle={IEEE Conference on Computer Vision and Pattern Recognition (CVPR)},
  year={2017}
}

@article{wortsman2020supermasks,
  title={Supermasks in superposition},
  author={Wortsman, Mitchell and Ramanujan, Vivek and Liu, Rosanne and
          Kembhavi, Aniruddha and Rastegari, Mohammad and Yosinski, Jason
          and Farhadi, Ali},
  journal={Advances in Neural Information Processing Systems},
  volume={33},
  pages={15173--15184},
  year={2020}
}

@inproceedings{serra2018overcoming,
  title={Overcoming catastrophic forgetting with hard attention to the task},
  author={Serra, Joan and Suris, Didac and Miron, Marius and Karatzoglou, Alexandros},
  booktitle={International Conference on Machine Learning},
  pages={4548--4557},
  year={2018},
  organization={PMLR}
}

@inproceedings{mallya2018packnet,
  title={PackNet: Adding multiple tasks to a single network by iterative pruning},
  author={Mallya, Arun and Lazebnik, Svetlana},
  booktitle={IEEE Conference on Computer Vision and Pattern Recognition (CVPR)},
  year={2018}
}

@article{rusu2016progressive,
  title={Progressive neural networks},
  author={Rusu, Andrei A and Rabinowitz, Neil C and Desjardins, Guillaume and
          Soyer, Hubert and Kirkpatrick, James and Kavukcuoglu, Koray and
          Pascanu, Razvan and Hadsell, Raia},
  journal={arXiv preprint arXiv:1606.04671},
  year={2016}
}

@article{yoon2018lifelong,
  title={Lifelong learning with dynamically expandable networks},
  author={Yoon, Jaehong and Yang, Eunho and Lee, Jeongtae and Hwang, Sung Ju},
  journal={ICLR},
  year={2018}
}

@article{yang2019task,
  title={Task representations in neural networks trained to perform many
         cognitive tasks},
  author={Yang, Guangyu Robert and Joglekar, Madhura R and Song, H Francis
          and Newsome, William T and Wang, Xiao-Jing},
  journal={Nature neuroscience},
  volume={22},
  number={2},
  pages={297--306},
  year={2019},
  publisher={Nature Publishing Group}
}

@article{mongillo2018inhibitory,
  title={Inhibitory connectivity defines the realm of excitatory plasticity},
  author={Mongillo, Gianluigi and Rumpel, Simon and Loewenstein, Yonatan},
  journal={Nature neuroscience},
  volume={21},
  number={10},
  pages={1463--1470},
  year={2018},
  publisher={Nature Publishing Group}
}

@article{cichon2015branch,
  title={Branch-specific dendritic Ca2+ spikes cause persistent synaptic
         plasticity},
  author={Cichon, Joseph and Gan, Wen-Biao},
  journal={Nature},
  volume={520},
  number={7546},
  pages={180--185},
  year={2015},
  publisher={Nature Publishing Group}
}

@article{kohonen1982self,
  title={Self-organized formation of topologically correct feature maps},
  author={Kohonen, Teuvo},
  journal={Biological cybernetics},
  volume={43},
  number={1},
  pages={59--69},
  year={1982},
  publisher={Springer}
}

@article{poirazi2020illuminating,
  title={Illuminating dendritic function with computational models},
  author={Poirazi, Panayiota and Papoutsi, Athanasia},
  journal={Nature Reviews Neuroscience},
  volume={21},
  number={6},
  pages={303--321},
  year={2020},
  publisher={Nature Publishing Group}
}

@article{beniaguev2021single,
  title={Single cortical neurons as deep artificial neural networks},
  author={Beniaguev, David and Segev, Idan and London, Michael},
  journal={Neuron},
  volume={109},
  number={17},
  pages={2727--2739},
  year={2021},
  publisher={Elsevier}
}

@article{maass2000computational,
  title={On the computational power of winner-take-all},
  author={Maass, Wolfgang},
  journal={Neural computation},
  volume={12},
  number={11},
  pages={2519--2535},
  year={2000},
  publisher={MIT Press}
}

@article{douglas2004neuronal,
  title={Neuronal circuits of the neocortex},
  author={Douglas, Rodney J and Martin, Kevan AC},
  journal={Annual Review of Neuroscience},
  volume={27},
  pages={419--451},
  year={2004}
}

@article{hubel1962receptive,
  title={Receptive fields, binocular interaction and functional architecture in the cat's visual cortex},
  author={Hubel, David H and Wiesel, Torsten N},
  journal={The Journal of Physiology},
  volume={160},
  number={1},
  pages={106--154},
  year={1962}
}

@article{mink1996basal,
  title={The basal ganglia: focused selection and inhibition of competing motor programs},
  author={Mink, Jonathan W},
  journal={Progress in Neurobiology},
  volume={50},
  number={4},
  pages={381--425},
  year={1996}
}

@article{schultz1997neural,
  title={A neural substrate of prediction and reward},
  author={Schultz, Wolfram and Dayan, Peter and Montague, P Read},
  journal={Science},
  volume={275},
  number={5306},
  pages={1593--1599},
  year={1997}
}

@article{rolls2010attractor,
  title={Attractor networks},
  author={Rolls, Edmund T},
  journal={Wiley Interdisciplinary Reviews: Cognitive Science},
  volume={1},
  number={1},
  pages={119--134},
  year={2010}
}

@article{srivastava2014dropout,
  title={Dropout: a simple way to prevent neural networks from overfitting},
  author={Srivastava, Nitish and Hinton, Geoffrey and Krizhevsky, Alex and Sutskever, Ilya and Salakhutdinov, Ruslan},
  journal={Journal of Machine Learning Research},
  volume={15},
  number={1},
  pages={1929--1958},
  year={2014}
}

@article{shazeer2017outrageously,
  title={Outrageously large neural networks: The sparsely-gated mixture-of-experts layer},
  author={Shazeer, Noam and Mirhoseini, Azalia and Maziarz, Krzysztof and Davis, Andy and Le, Quoc and Hinton, Geoffrey and Dean, Jeff},
  journal={ICLR},
  year={2017}
}

@article{fedus2022switch,
  title={Switch transformers: Scaling to trillion parameter models with simple and efficient sparsity},
  author={Fedus, William and Zoph, Barret and Shazeer, Noam},
  journal={Journal of Machine Learning Research},
  volume={23},
  pages={1--39},
  year={2022}
}

@article{aljundi2019taskfree,
  title={Task-free continual learning},
  author={Aljundi, Rahaf and Kelchtermans, Klaas and Tuytelaars, Tinne},
  journal={IEEE CVPR},
  year={2019}
}

@article{gama2014concept,
  title={A survey on concept drift adaptation},
  author={Gama, Jo{\~a}o and {\v{Z}}liobait{\.e}, Indr{\.e} and Bifet, Albert and Pechenizkiy, Mykola and Bouchachia, Abdelhamid},
  journal={ACM Computing Surveys},
  volume={46},
  number={4},
  pages={1--37},
  year={2014}
}

@article{parisi2019continual,
  title={Continual lifelong learning with neural networks: A review},
  author={Parisi, German I and Kemker, Ronald and Part, Jose L and Kanan, Christopher and Wermter, Stefan},
  journal={Neural Networks},
  volume={113},
  pages={54--71},
  year={2019}
}

@inproceedings{gururangan2020dontstop,
  title={Don't Stop Pretraining: Adapt Language Models to Domains and Tasks},
  author={Gururangan, Suchin and Marasovi{\'c}, Ana and Swayamdipta, Swabha and Lo, Kyle and Beltagy, Iz and Downey, Doug and Smith, Noah A.},
  booktitle={Proceedings of the 58th Annual Meeting of the Association for Computational Linguistics},
  pages={8342--8360},
  year={2020},
  publisher={Association for Computational Linguistics}
}

@article{yildiz2024investigating,
  title={Investigating Continual Pretraining in Large Language Models: Insights and Implications},
  author={Y{\i}ld{\i}z, {\c{C}}a{\u{g}}atay and Ravichandran, Nishaanth Kanna and Sharma, Nitin and Bethge, Matthias and Ermis, Beyza},
  journal={arXiv preprint arXiv:2402.17400},
  year={2024}
}

@book{quinonero2009dataset,
  title={Dataset Shift in Machine Learning},
  editor={Qui{\~n}onero-Candela, Joaquin and Sugiyama, Masashi and Schwaighofer, Anton and Lawrence, Neil D.},
  publisher={MIT Press},
  year={2009}
}

@article{morenotorres2012unifying,
  title={A unifying view on dataset shift in classification},
  author={Moreno-Torres, Jose G. and Raeder, Troy and Alaiz-Rodr{\'i}guez, Roc{\'i}o and Chawla, Nitesh V. and Herrera, Francisco},
  journal={Pattern Recognition},
  volume={45},
  number={1},
  pages={521--530},
  year={2012},
  publisher={Elsevier}
}

@article{van2019three,
  title={Three scenarios for continual learning},
  author={van de Ven, Gido M. and Tolias, Andreas S.},
  journal={arXiv preprint arXiv:1904.07734},
  year={2019}
}

@article{goodfellow2013empirical,
  title={An empirical investigation of catastrophic forgetting in gradient-based neural networks},
  author={Goodfellow, Ian J. and Mirza, Mehdi and Xiao, Da and Courville, Aaron and Bengio, Yoshua},
  journal={arXiv preprint arXiv:1312.6211},
  year={2013}
}

@article{poirazi2003pyramidal,
  title={Pyramidal neuron as two-layer neural network},
  author={Poirazi, Panayiota and Brannon, Terrence and Mel, Bartlett W.},
  journal={Neuron},
  volume={37},
  number={6},
  pages={989--999},
  year={2003},
  publisher={Elsevier}
}

@article{gilbert1983clustered,
  title={Clustered intrinsic connections in cat visual cortex},
  author={Gilbert, Charles D. and Wiesel, Torsten N.},
  journal={Journal of Neuroscience},
  volume={3},
  number={5},
  pages={1116--1133},
  year={1983}
}

@article{amari1977dynamics,
  title={Dynamics of pattern formation in lateral-inhibition type neural fields},
  author={Amari, Shun-ichi},
  journal={Biological Cybernetics},
  volume={27},
  number={2},
  pages={77--87},
  year={1977},
  publisher={Springer}
}

@article{wilson1972excitatory,
  title={Excitatory and inhibitory interactions in localized populations of model neurons},
  author={Wilson, Hugh R. and Cowan, Jack D.},
  journal={Biophysical Journal},
  volume={12},
  number={1},
  pages={1--24},
  year={1972},
  publisher={Elsevier}
}

@article{khona2022attractor,
  title={Attractor and integrator networks in the brain},
  author={Khona, Mikail and Fiete, Ila R.},
  journal={Nature Reviews Neuroscience},
  volume={23},
  number={12},
  pages={744--766},
  year={2022},
  publisher={Nature Publishing Group}
}

@article{frank2001interactions,
  title={Interactions between frontal cortex and basal ganglia in working memory: A computational model},
  author={Frank, Michael J. and Loughry, Bryan and O'Reilly, Randall C.},
  journal={Cognitive, Affective, \& Behavioral Neuroscience},
  volume={1},
  number={2},
  pages={137--160},
  year={2001},
  publisher={Springer}
}

@article{oreilly2006making,
  title={Making working memory work: A computational model of learning in the prefrontal cortex and basal ganglia},
  author={O'Reilly, Randall C. and Frank, Michael J.},
  journal={Neural Computation},
  volume={18},
  number={2},
  pages={283--328},
  year={2006},
  publisher={MIT Press}
}

@article{faisal2008noise,
  title={Noise in the nervous system},
  author={Faisal, A. Aldo and Selen, Luc P. J. and Wolpert, Daniel M.},
  journal={Nature Reviews Neuroscience},
  volume={9},
  number={4},
  pages={292--303},
  year={2008},
  publisher={Nature Publishing Group}
}

@article{mainen1995reliability,
  title={Reliability of spike timing in neocortical neurons},
  author={Mainen, Zachary F. and Sejnowski, Terrence J.},
  journal={Science},
  volume={268},
  number={5216},
  pages={1503--1506},
  year={1995},
  publisher={American Association for the Advancement of Science}
}

@article{natarajan1995sparse,
  title={Sparse approximate solutions to linear systems},
  author={Natarajan, Balas Kausik},
  journal={SIAM Journal on Computing},
  volume={24},
  number={2},
  pages={227--234},
  year={1995},
  publisher={SIAM}
}

\appendix

\section{Per-method Performance Matrices: Stored vs Recovered}
\label{app:perf_matrices}

This appendix collects the full $3{\times}3$ performance matrices for Experiment~1 across the four headline mask configurers, NoMask (the null mask reuses the same $128$-slot subnetwork for every task), KWTA-only, FTN-Fast, and FTN-Slow, under both evaluation protocols. Each panel shows two heatmaps side by side: the left heatmap (\emph{stored}) evaluates each finished training block using the mask that was active when that task was trained; the right heatmap (\emph{recovered}) evaluates with a mask that the configurer is asked to find from scratch on a fresh evaluation batch, with no task identity given. Cell $(i,j)$ is the average performance on task~$j$ after training through task~$i$, averaged over $8$ seeds. Classification cells report ACC (range $[0,1]$, higher is better); regression cells report MSE clamped at $0.5$ for visual comparability (lower is better; cells with mean MSE above $0.5$ render in the saturated top color and are flagged with the colorbar's $\rightarrow$ extension).

\paragraph{Reading the panels.}
The two heatmaps separate two failure modes that get summed together in the headline ACC/MSE:

\begin{itemize}
  \item A drop in the \emph{stored} (left) panel along an off-diagonal row $i$, column $j<i$ measures how much training task~$j$'s solution has been \emph{overwritten} by subsequent gradient updates on tasks~$j{+}1\ldots i$. With perfectly disjoint masks this drop is zero by construction (the FixedMask appendix entry, omitted here for space, has identical diagonal and off-diagonal entries; cf.\ Theorem~\ref{thm:isolation}). With the optimizer state reset described in Section~\ref{sec:setup}, FixedMask attains $\mathrm{FM}=0.000\pm0.000$ on Exp.\ 1 (Table~\ref{tab:exp1}) so any drop here is attributable to \emph{shared neurons between tasks at training time}. Methods with non-disjoint masks, including NoMask (full overlap by construction) and any adaptive method whose discovered masks happen to overlap, will show negative gradients on the stored column.
  \item A drop from the \emph{stored} panel to the corresponding cell in the \emph{recovered} panel measures \emph{mask recall error}: the underlying weights are unchanged, only the mask is being chosen by the configurer, so any difference is entirely attributable to picking the wrong subnetwork at evaluation time. NoMask and EWC have no recovery problem because they always use the same all-on-128 mask, but they pay this cost in the stored column. KWTA-only sits between the two: it writes moderately clean training masks but its one-step gradient proposal does not always converge to the same neuron set on a fresh batch, so the stored panel is noticeably better than the recovered panel; the regression matrices in particular show a substantial drop from stored to recovered.
\end{itemize}

The reader should therefore compare the two panels jointly. A method that loses 5\% on the stored column has an irrecoverable problem (gradient conflict between tasks); a method that loses 5\% from stored to recovered has a soluble problem (the trained subnetwork still exists, only the lookup is noisy). The aggregate decomposition was already given in Table~\ref{tab:overlap_recall} in the main text; the matrices below show where each loss is concentrated within the task sequence.

\paragraph{Per-method matrices, classification.}

\begin{figure}[h]
\centering
\begin{subfigure}[b]{0.48\textwidth}
  \includegraphics[width=\textwidth]{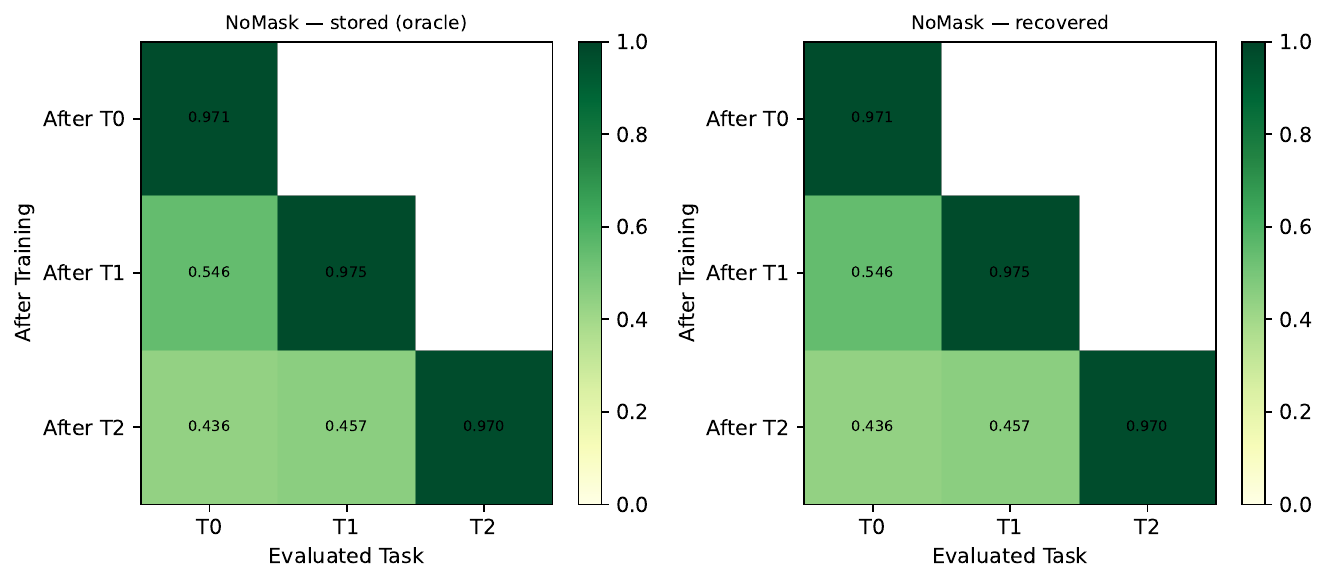}
  \caption{NoMask (clf): the same fixed first-128 mask is used for every task, so the stored panel shows full overwriting (rows~1--2 zero out task~0) and the recovered panel is identical because no recovery happens.}
\end{subfigure}
\hfill
\begin{subfigure}[b]{0.48\textwidth}
  \includegraphics[width=\textwidth]{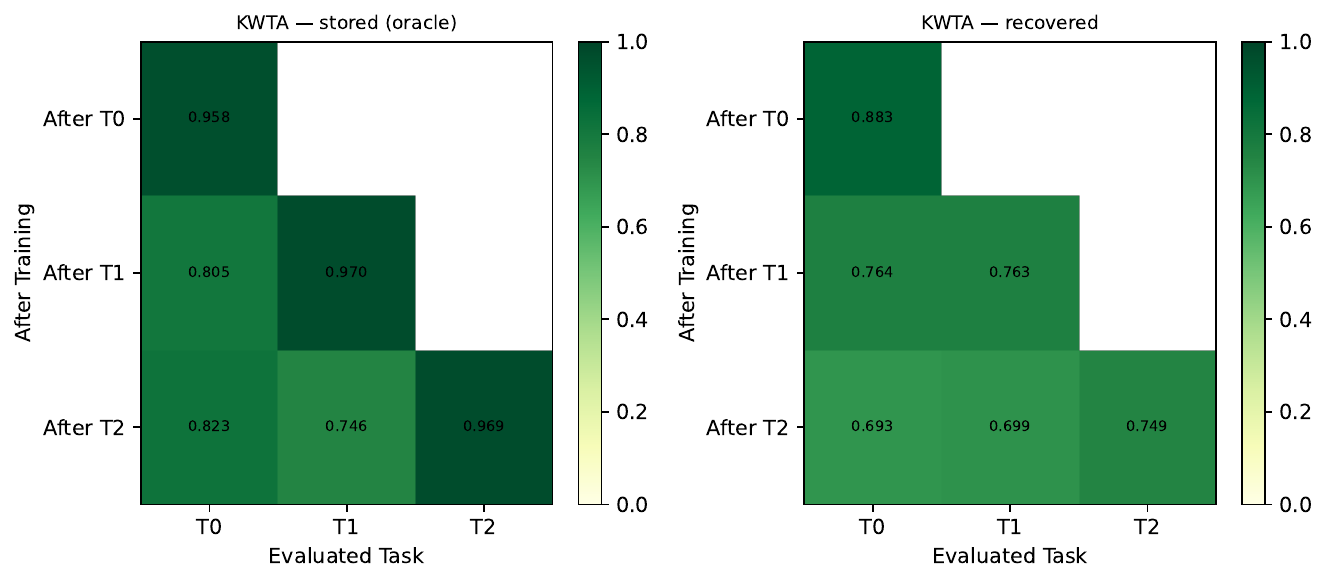}
  \caption{KWTA-only (clf): the stored panel is close to the FixedMask oracle, indicating the trained masks were nearly disjoint, but the recovered panel still loses about $0.13$~ACC on prior tasks because a one-step gradient proposal with no spatial averaging picks an inconsistent mask at evaluation time.}
\end{subfigure}\\[1ex]
\begin{subfigure}[b]{0.48\textwidth}
  \includegraphics[width=\textwidth]{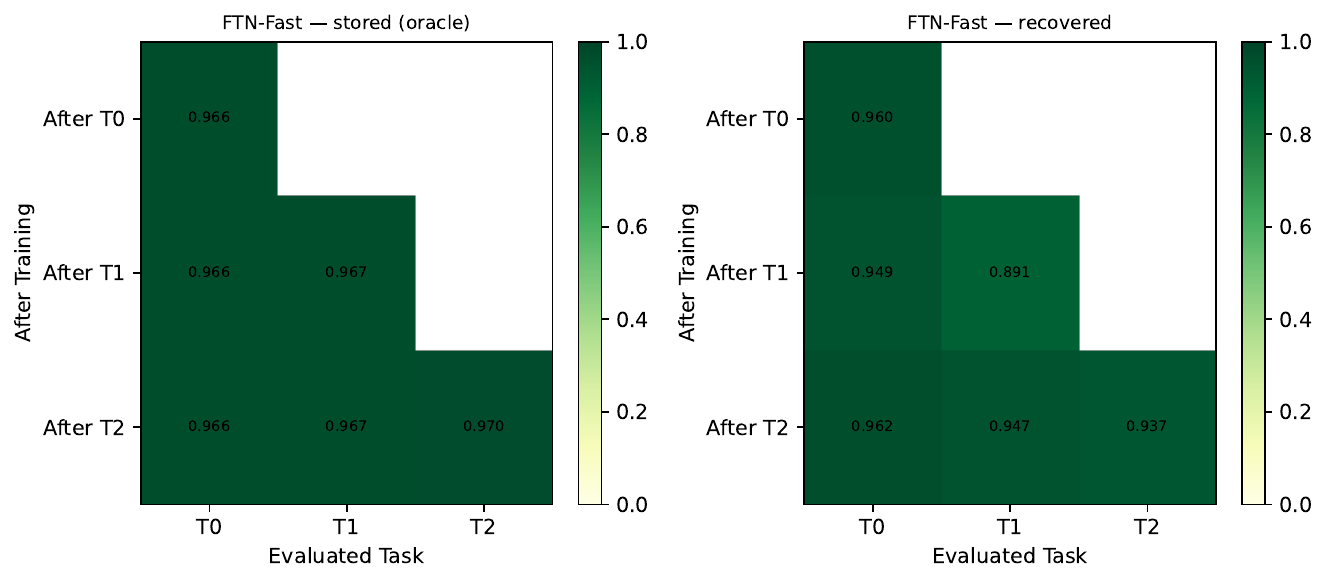}
  \caption{FTN-Fast (clf): stored and recovered panels are nearly identical. The large $17\times17$ kernel with two iterations is enough to make recovery deterministic at the level needed for accuracy.}
\end{subfigure}
\hfill
\begin{subfigure}[b]{0.48\textwidth}
  \includegraphics[width=\textwidth]{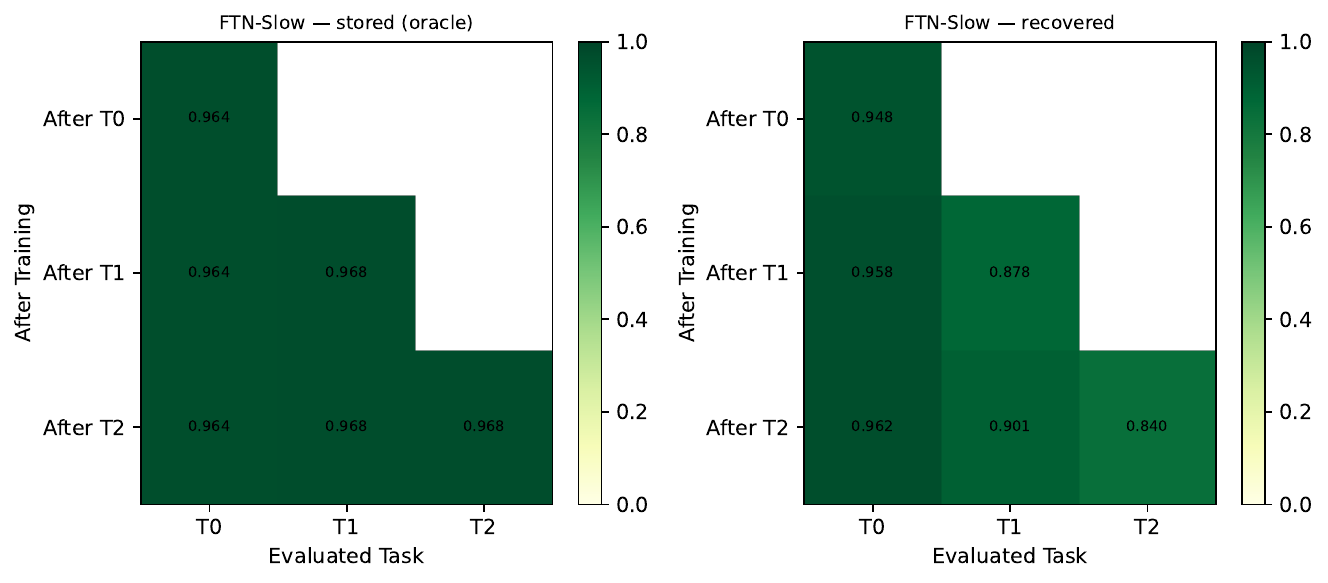}
  \caption{FTN-Slow (clf): the smallest training-time overlap of any adaptive method (cleanest stored panel) but slightly larger recall noise than FTN-Fast, a small kernel applied for many steps produces the tightest spatial separation, at the cost of slightly higher recovery variance per call.}
\end{subfigure}
\caption{Stored vs.\ recovered $3{\times}3$ performance matrices on synthetic classification (mean ACC over 8 seeds; range $[0,1]$). Cell $(i,j)$ is performance on task $j$ after training through task $i$.}
\label{fig:app_perf_clf}
\end{figure}

\paragraph{Per-method matrices, regression.}

\begin{figure}[h]
\centering
\begin{subfigure}[b]{0.48\textwidth}
  \includegraphics[width=\textwidth]{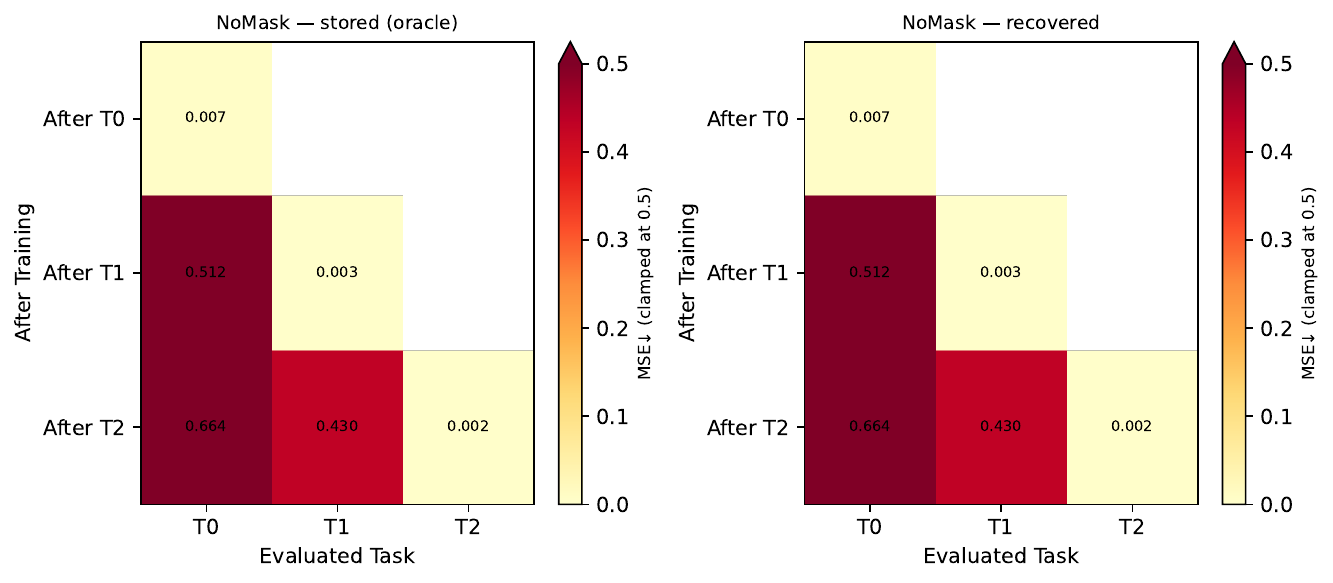}
  \caption{NoMask (reg): off-diagonal MSE saturates at the colorbar cap, mirroring the classification picture, the single shared subnetwork cannot retain incompatible regression targets.}
\end{subfigure}
\hfill
\begin{subfigure}[b]{0.48\textwidth}
  \includegraphics[width=\textwidth]{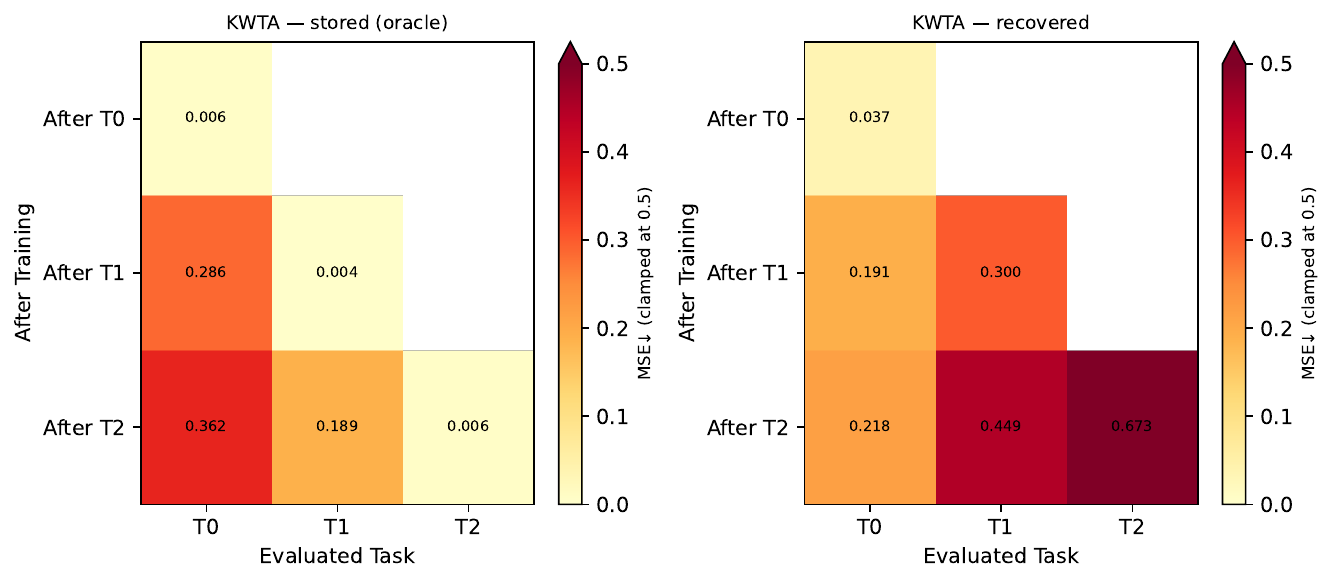}
  \caption{KWTA-only (reg): with $S{=}10$ steps at $\eta_m{=}0.2$, the regression configurer has more time per call than the classification one ($S{=}1$ at $\eta_m{=}1.0$), but the absence of lateral smoothing still produces a substantial stored-to-recovered drop on prior tasks ($+0.27$ MSE excess); both training-time overlap and recall error contribute roughly equally to the residual MSE.}
\end{subfigure}\\[1ex]
\begin{subfigure}[b]{0.48\textwidth}
  \includegraphics[width=\textwidth]{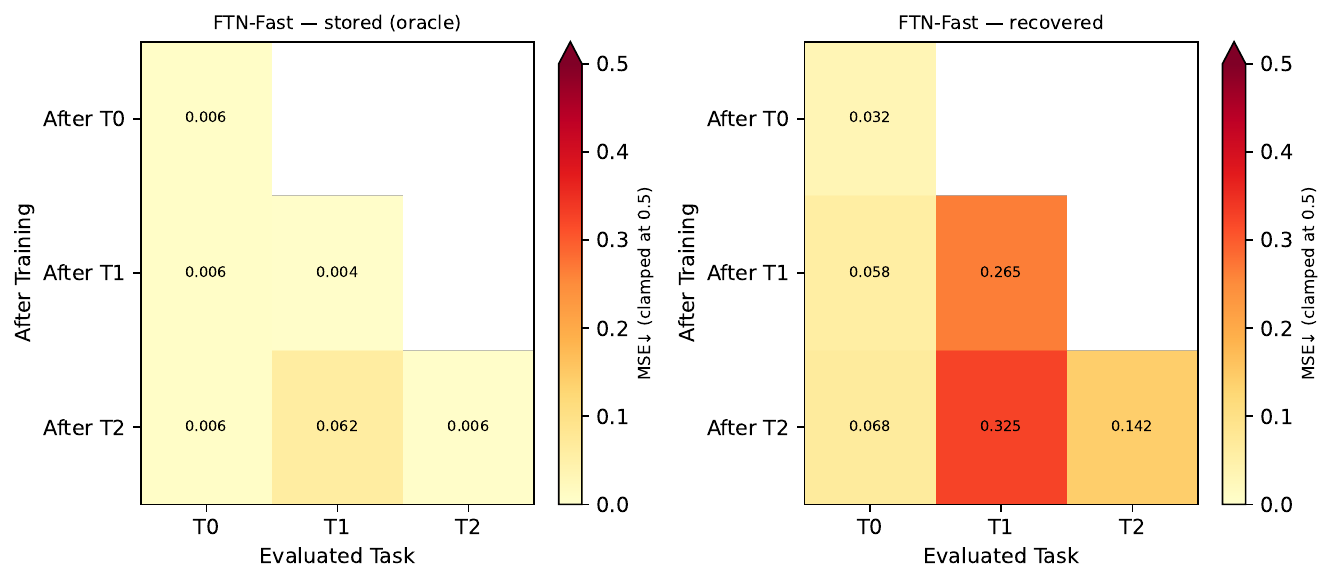}
  \caption{FTN-Fast (reg): stored and recovered matrices both stay close to the FixedMask oracle (max prior-task MSE around $0.13$).}
\end{subfigure}
\hfill
\begin{subfigure}[b]{0.48\textwidth}
  \includegraphics[width=\textwidth]{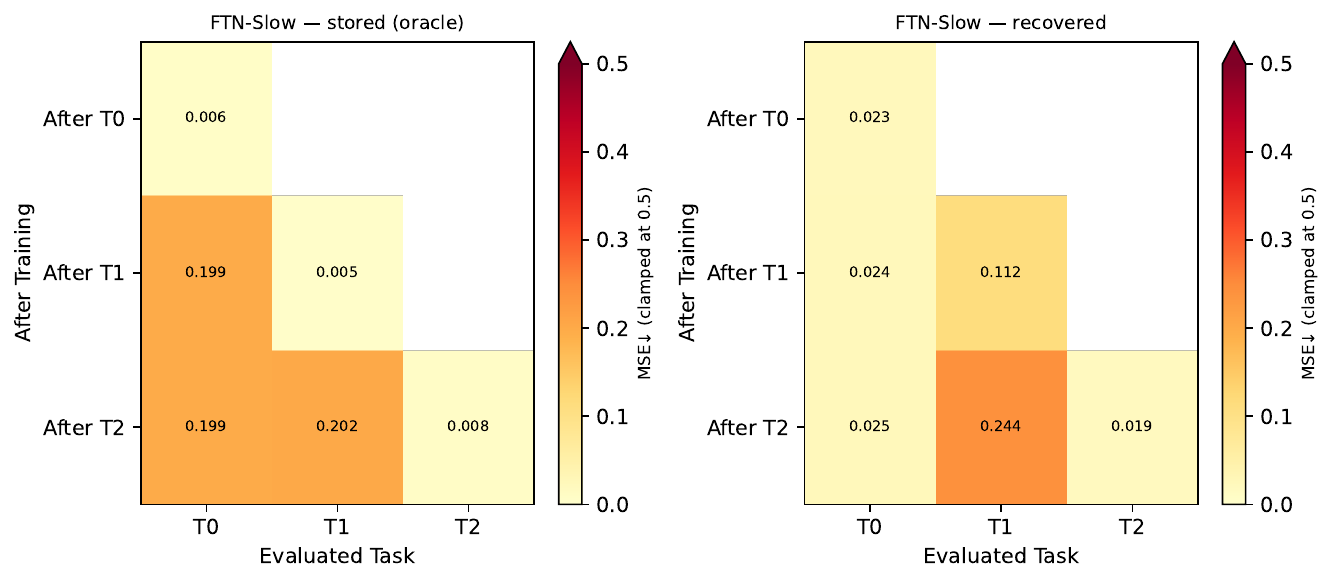}
  \caption{FTN-Slow (reg): smallest off-diagonal MSE among adaptive methods on both panels; the fine-grained spatial diffusion produces the tightest neuron-set separation across tasks, and the 10-step regression configurer recovers it accurately.}
\end{subfigure}
\caption{Stored vs.\ recovered $3{\times}3$ performance matrices on synthetic regression (mean MSE over 8 seeds; clamped at $0.5$, lower is better, $\rightarrow$ on the colorbar denotes saturation).}
\label{fig:app_perf_reg}
\end{figure}

\paragraph{Why classification and regression look different.}
The two configurers differ only in $S$ (1 vs.\ 10) and $\eta_m$ (1.0 vs.\ 0.2); architecture, training data, optimizer, capacity, and the zero mask cold-start are identical (cf.\ Section~\ref{sec:exp1}). On classification, a single high-LR step gives a usable but batch-noisy proposal whose support shifts between calls. FTN's lateral kernel averages that noise into a stable spatial blob, so stored and recovered masks coincide; KWTA-only has no smoothing to stabilize the proposal across calls and therefore picks an inconsistent mask at evaluation time, producing the $0.131$ ACC recall gap visible in Figure~\ref{fig:app_perf_clf} (top right). On regression, ten gentler steps converge to a more reproducible proposal, but the higher overall MSE scale of regression amplifies any neuron-set inconsistency: KWTA-only's recall error here ($+0.267$ MSE excess) is in fact the largest single contributor to its regression loss, and FTN's lateral smoothing is what brings recall error down by an order of magnitude.

\end{document}